\documentclass{article} 
\usepackage{nips15submit_e,times}
\usepackage{hyperref}
\usepackage{url}

\usepackage{graphicx} 
\usepackage{subfigure}

\usepackage{booktabs}
\usepackage{array}
\newcolumntype{L}[1]{>{\raggedright\let\newline\\\arraybackslash\hspace{0pt}}m{#1}}
\newcolumntype{C}[1]{>{\centering\let\newline\\\arraybackslash\hspace{0pt}}m{#1}}
\newcolumntype{R}[1]{>{\raggedleft\let\newline\\\arraybackslash\hspace{0pt}}m{#1}}

\usepackage{algorithm}
\usepackage{hyperref}
\usepackage{amsmath,  amssymb, amsthm}

\usepackage{color}

\usepackage{lipsum}
\usepackage{algorithm}
\usepackage{algpseudocode}

\newcommand{\beq}{\vspace{0mm}\begin{equation}}
\newcommand{\eeq}{\vspace{0mm}\end{equation}}
\newcommand{\beqs}{\vspace{0mm}\begin{eqnarray}}
\newcommand{\eeqs}{\vspace{0mm}\end{eqnarray}}
\newcommand{\barr}{\begin{array}}
\newcommand{\earr}{\end{array}}

\newcommand{\Imat}{{\bf I}}

\newcommand{\mv}[0]{{\boldsymbol{m}}}

\newcommand{\rv}{\boldsymbol{r}}

\newcommand{\xv}{\boldsymbol{x}}

\newcommand{\cdotv}{\boldsymbol{\cdot}}

\newcommand{\Phimat}{\boldsymbol{\Phi}}

\newcommand{\thetav}{\boldsymbol{\theta}}

\newcommand{\phiv}{\boldsymbol{\phi}}

\newcommand{\E}{\mathbb{E}}

\newtheorem{thm}{Theorem} 
\newtheorem{cor}[thm]{Corollary}
\newtheorem{lem}[thm]{Lemma}

\nipsfinalcopy 

\title{
The Poisson Gamma Belief Network
}

\author{
Mingyuan Zhou 
\\
McCombs School of Business\\
The University of Texas at Austin\\
Austin, TX 78712, USA \\
\And
Yulai Cong \\
\!School of Elec. Engineering\! \\ 
Xidian University\\
Xi'an, Shaanxi, 
China \\
\And
Bo Chen \\
School of Elec. Engineering \\ 
Xidian University\\
Xi'an, Shaanxi, 
China \\
}

\begin{document}

\maketitle

\begin{abstract}
To infer a multilayer representation of high-dimensional count vectors, we propose the Poisson gamma belief network (PGBN) that factorizes each of its layers into the product of a connection weight matrix and the nonnegative real hidden units of the next layer. The PGBN's hidden layers are jointly trained with an upward-downward Gibbs sampler, each iteration of which upward samples Dirichlet distributed connection weight vectors starting from the first layer (bottom data layer), and then downward samples gamma distributed hidden units starting from the top hidden layer. The gamma-negative binomial process combined with a layer-wise training strategy allows the PGBN to infer the width of each layer given a fixed budget on the width of the first layer. The PGBN with a single hidden layer reduces to Poisson factor analysis. Example results on text analysis illustrate interesting relationships between the width of the first layer and the inferred network structure, and demonstrate that the PGBN, whose hidden units are imposed with correlated gamma priors, can add more layers to increase its performance gains over Poisson factor analysis, given the same limit on the width of the first layer. 

\end{abstract}

\section{Introduction}

There has been significant recent interest in deep learning. Despite its tremendous success in supervised learning,  inferring a multilayer data representation in an unsupervised manner remains a challenging problem \cite{Bengio+chapter2007,ranzato2007unsupervised,Bengio-et-al-2015-Book}. The sigmoid belief network (SBN), which connects the  binary units of adjacent layers 
 via the sigmoid functions, infers a deep representation of multivariate binary vectors \cite{neal1992connectionist,saul1996mean}. The deep belief network (DBN) \cite{hinton2006fast} is a SBN whose top hidden layer is replaced by the restricted Boltzmann machine (RBM) \cite{POE} that is undirected. The deep Boltzmann machine (DBM) is an undirected deep network that connects the binary units of adjacent layers using the RBMs \cite{salakhutdinov2009deep}.  All these deep networks are designed to model binary observations. Although one may modify the bottom layer 
 to model Gaussian and multinomial observations,  the hidden units of these networks are still typically restricted to be binary \cite{salakhutdinov2009deep,larochelle2012neural,salakhutdinov2013learning}. 
  One may further consider the exponential family harmoniums \cite{welling2004exponential,xing2005mining} to construct more general networks with non-binary hidden units, but often at the expense of noticeably increased complexity in training and data fitting. 
 
 Moving beyond conventional deep networks using binary hidden units,
we construct a deep directed network with gamma distributed nonnegative real hidden units to unsupervisedly infer a multilayer representation of multivariate count vectors, with a simple but powerful mechanism to capture the correlations among the visible/hidden features across all layers and handle  highly overdispersed counts. 
The proposed model is called the Poisson gamma belief network (PGBN), which   factorizes the observed count vectors under the Poisson likelihood into the product of a factor loading matrix and the gamma distributed hidden units (factor scores) of layer one; and further factorizes the shape parameters of the gamma hidden units of each layer   into the product of a connection weight matrix  and the gamma  hidden units of the next layer.  
Distinct from previous deep networks  that often utilize binary units for tractable inference and require tuning both the width (number of hidden units) of each layer and the network depth (number of layers), the PGBN employs nonnegative real  hidden units and automatically infers the widths of subsequent layers given a fixed budget on the width of its first layer.
Note that the budget could be infinite and hence the whole network can  grow without bound as more data are being observed. When the budget is finite and hence the ultimate capacity of the network is limited, we find that the PGBN equipped with a narrower first layer could increase its depth to match or even outperform a shallower network with a substantially wider first layer.

The gamma distribution density function has the highly desired strong non-linearity  for deep learning, 
but the existence of neither a conjugate prior nor a closed-form maximum likelihood estimate 
for its shape parameter makes a deep network with gamma hidden units appear unattractive. Despite seemingly difficult, we discover that, by generalizing the  data augmentation and marginalization techniques  for discrete data \cite{NBP2012},  one may propagate latent counts one layer at a time  from  the bottom data layer to the top hidden layer, with which one may derive an efficient upward-downward Gibbs sampler that, 
one layer at a time in each iteration, upward samples Dirichlet distributed connection weight vectors and then downward samples gamma distributed hidden units.  

In addition to constructing a new deep network that well fits multivariate count data and developing an efficient upward-downward Gibbs sampler, other contributions of the paper include: 1) combining the gamma-negative binomial process \cite{NBP2012,NBP_CountMatrix} with a layer-wise training strategy to automatically infer the network structure; 2) revealing the relationship between the upper bound imposed on the width of the first  layer and the inferred widths of subsequent layers; 3) revealing the relationship between the network depth and the model's ability to model overdispersed counts; 
4) and generating  a multivariate high-dimensional  random count vector, whose distribution is governed by the PGBN, 
 by propagating  
the gamma hidden units of
the top hidden layer back  to the bottom data layer. 

\subsection{Useful count distributions and their relationships}

Let the Chinese restaurant table (CRT) distribution $l\sim\mbox{CRT}(n,r)$ represent the distribution of a random count generated as 
$
l=\sum_{i=1}^{n} b_i,~b_i\sim\mbox{Bernoulli}\left[{r}/{(r+i-1)}\right].\notag
$ Its probability mass function (PMF) can be expressed as $
 P(l\,|\,n,r) = \frac{\Gamma(r)r^l}{\Gamma(n+r)}|s(n,l)|,\notag
$
where $l\in\mathbb{Z}$, $\mathbb{Z}:=\{0,1,\ldots,n\}$, and $|s(n,l)|$ 
 are unsigned Stirling numbers of the first kind. Let $u\sim\mbox{Log}(p)$ denote the logarithmic distribution with PMF
$
P(u\,|\,p) = \frac{1}{-\ln(1-p)}\frac{p^u}{u},\notag
$
where $u\in\{1,2,\ldots\}$. Let $n\sim\mbox{NB}(r,p)$ denote the negative binomial (NB) distribution with PMF
$
P(n\,|\,r,p)=\frac{\Gamma(n+r)}{n!\Gamma(r)} p^n(1-p)^r,\notag
$
where $n\in\mathbb{Z}$. 
The NB distribution $n\sim\mbox{NB}(r,p)$ can be generated as a gamma mixed Poisson distribution as
$
n\sim\mbox{Pois}(\lambda),~\lambda\sim\mbox{Gam}\left[r,{p}/({1-p})\right],
$
where $p/(1-p)$ is the gamma  scale parameter. 
As shown in  \cite{NBP2012},
the joint distribution of $n$ and $l$ given $r$ and $p$ in 
$
l\sim\mbox{CRT}(n,r),~n\sim\mbox{NB}(r,p),\notag
$
where $l\in\{0,\ldots,n\}$ and $n\in\mathbb{Z}$, is the same as that in
$
n = \textstyle \sum_{t=1}^l u_t,~u_t\sim\mbox{Log}(p), ~l\sim\mbox{Pois}[-r\ln(1-p)], \notag
$
which is called the Poisson-logarithmic bivariate distribution, with PMF 
$
P(n,l\,|\,r,p)=\frac{|s(n,l)|r^l}{n!}p^n(1-p)^r\label{eq:Po-log}
$.

\section{The Poisson Gamma Belief Network}

Assuming the observations are multivariate count vectors $\xv_j^{(1)}\in\mathbb{Z}^{K_{0}}$, the generative model of the Poisson gamma belief network (PGBN) with $T$ hidden layers, from top to bottom, 
 is expressed as
\beqs\displaystyle
&\thetav_j^{(T)}\sim\mbox{Gam}\left(\rv,1\big/c_j^{(T+1)}\right),\notag\\ 
&\cdots\notag\\
&\thetav_j^{(t)}\sim\mbox{Gam}\left(\Phimat^{(t+1)}\thetav_j^{(t+1)},1\big/c_j^{(t+1)}\right),\notag\\
&\cdots\notag\\
&\xv_j^{(1)} \sim \mbox{Pois}\left(\Phimat^{(1)}\thetav_j^{(1)}\right),~~\thetav_j^{(1)}\sim\mbox{Gam}\left(\Phimat^{(2)}\thetav_j^{(2)},
{p_j^{(2)}}\big/{\big(1-p_j^{(2)}\big)}\right). 
 \label{eq:PGBN}
\eeqs
The PGBN factorizes the count observation  $\xv_j^{(1)}$  into  the product of the factor loading $\Phimat^{(1)}\in\mathbb{R}_+^{K_{0}\times K_{1}}$
and  hidden units $\thetav_j^{(1)}\in\mathbb{R}_+^{ K_{1}}$ of layer one under the Poisson likelihood, where $ \mathbb{R}_+=\{x:x\ge 0\}$, and for $t=1,2,\ldots,T-1$, factorizes the shape parameters of the gamma distributed hidden units $\thetav_j^{(t)}\in\mathbb{R}_+^{K_{t}}$ of layer $t$  into the product of the connection weight matrix $\Phimat^{(t+1)}\in\mathbb{R}_+^{K_{t}\times K_{t+1}}$ and the hidden units $\thetav_j^{(t+1)}\in\mathbb{R}_+^{ K_{t+1}}$ of layer $t+1$; the top layer's hidden units $\thetav_j^{(T)}$ share the same vector $\rv=(r_1,\ldots,r_{K_T})' $ as their gamma shape parameters; and the $p_j^{(2)}$ are probability parameters and 
$\{1/c^{(t)}\}_{3,T+1}$ are gamma scale parameters, with $c_j^{(2)}:=\big(1-p_j^{(2)}\big)\big/p_j^{(2)}$.

For scale identifiabilty and ease of inference,  each column of $\Phimat^{(t)}\in\mathbb{R}_+^{K_{t-1}\times K_{t}}$ is restricted to have a unit $L_1$ norm. 
To complete the hierarchical model, for $t\in\{1,\ldots,T-1\}$, we let 
\beqs
&\phiv_k^{(t)}\sim\mbox{Dir}\big(\eta^{(t)},\ldots,\eta^{(t)}\big), ~~r_k\sim\mbox{Gam}\big(\gamma_0/K_T,1/c_0\big) \label{eq:Phi}
\eeqs 
and impose $c_0\sim\mbox{Gam}(e_0,1/f_0)$ 
and $\gamma_0\sim\mbox{Gam}(a_0,1/b_0)$; 
and for $t\in\{3,\ldots,T+1\}$, we let
\beqs
&p_j^{(2)}\sim\mbox{Beta}(a_0,b_0),~~~c_j^{(t)}\sim\mbox{Gam}(e_0,1/f_0). \label{eq:c_j}
\eeqs
We expect the correlations between the rows (features) of $(\xv_1^{(1)},\ldots,\xv_J^{(1)})$ to be captured by the columns of $\Phimat^{(1)}$, and the correlations between
the rows (latent features) of $(\thetav_1^{(t)}, \ldots,\thetav_J^{(t)})$ to be captured by the columns of $\Phimat^{(t+1)}$. Even if all $\Phimat^{(t)}$ for $t\ge 2$  are  identity matrices, indicating no correlations between latent features, our analysis will show that  a deep structure with $T\ge 2$ could still benefit data fitting by better modeling the variability  of the latent features $\thetav_j^{(1)}$. 


 \textbf{Sigmoid and deep belief networks}.
Under the hierarchical   model in (\ref{eq:PGBN}), given the connection weight matrices, 
 the joint distribution of the count observations and gamma hidden units of the PGBN can be expressed, similar to those of the sigmoid and deep belief networks \cite{Bengio-et-al-2015-Book},  as
 \beqs
&\!\!\!\!P\left(\xv_j^{(1)},\{\thetav_{j}^{(t)}\}_t  \,\Big|\,\{\Phimat^{(t)}\}_t\right) = P\left(\xv_j^{(1)}\,\Big|\,\Phimat^{(1)},\thetav_{j}^{(1)}\right)\! \left[\prod_{t=1}^{T-1}\!P\left(\thetav_{j}^{(t)}\,\Big|\,\Phimat^{(t+1)},\thetav_{j}^{(t+1)}\right)\right] \!P\left(\thetav_{j}^{(T)}\right).\notag
\eeqs
With $\phiv_{v:}$ representing the $v$th row  $\Phimat$, for the  gamma hidden units $\theta_{vj}^{(t)}$ we have  
\beqs
&P\left(\theta_{vj}^{(t)}\,\Big|\,\phiv_{v:}^{(t+1)},\thetav_j^{(t+1)},c_{j+1}^{(t+1)}\right) =\frac{ \left(c_{j+1}^{(t+1)}\right)^{\phiv_{v:}^{(t+1)}\thetav_j^{(t+1)} }}{\Gamma\left(\phiv_{v:}^{(t+1)}\thetav_j^{(t+1)}\right)}  \left(\theta_{vj}^{(t)}\right)^{\phiv_{v:}^{(t+1)}\thetav_j^{(t+1)}-1}e^{-c_{j+1}^{(t+1)} \theta_{vj}^{(t)} }, \label{eq:gamma}
\eeqs
which are highly nonlinear functions that are strongly desired in deep learning.  
By contrast, with the sigmoid function $\sigma(x)=1/(1+e^{-x})$ and bias terms $b_{v}^{(t+1)}$,  a sigmoid/deep belief network  would connect the binary hidden units $\theta_{vj}^{(t)}\in\{0,1\}$ of layer $t$ (for deep belief networks, $t<T-1$ )  to the product of the connection weights and binary hidden units of the next layer with
\beq
P\left(\theta_{vj}^{(t)}=1\,\big|\,\phiv_{v:}^{(t+1)},\thetav_j^{(t+1)},b_{v}^{(t+1)}\right) =\sigma\left(b_{v}^{(t+1)}+\phiv_{v:}^{(t+1)}\thetav_j^{(t+1)}\right). \label{eq:sigmoid}
\eeq
Comparing (\ref{eq:gamma}) with (\ref{eq:sigmoid}) clearly shows the differences between 
 the gamma nonnegative hidden units and the sigmoid link  based binary hidden units. Note that the rectified linear units have emerged as powerful alternatives of sigmoid units to introduce nonlinearity \cite{nair2010rectified}. It would be interesting to use the gamma units to  
 introduce nonlinearity in the positive region of the  rectified linear units.

\textbf{Deep Poisson factor analysis.} 
With $T=1$, the PGBN specified by (\ref{eq:PGBN})-(\ref{eq:c_j})
reduces to Poisson factor analysis (PFA) using the (truncated) gamma-negative binomial process \cite{NBP2012},  which is also related to latent Dirichlet allocation \cite{LDA} if the Dirichlet priors are imposed on both $\phiv_k^{(1)}$ and $\thetav_j^{(1)}$. 
With $T\ge 2$, the  PGBN is related to the gamma Markov chain hinted by Corollary~2 of \cite{NBP2012} and realized in  \cite{GP_DFA_AISTATS2015},  the deep exponential family of \cite{ranganath2014deep}, and the deep PFA of \cite{Gan2015DeepPFA}. 
Different from the PGBN, 
 in \cite{ranganath2014deep}, it is the gamma scale  but not shape parameters that are chained and factorized; 
in \cite{Gan2015DeepPFA}, it is the correlations between binary topic usage indicators but not the full connection weights that are captured; and neither   \cite{ranganath2014deep} nor \cite{Gan2015DeepPFA} 
provide  a 
principled way 
to learn the network structure. 
 Below we 
 break the PGBN of $T$ layers  into $T$ related submodels that are solved with the same subroutine.

\subsection{The propagation of latent counts and model properties}

\begin{lem} [Augment-and-conquer the PGBN]\label{lem:PGBN}
With $p_j^{(1)}: = 1-e^{-1}$ 
and
\beq
p_{j}^{(t+1)} := {-\ln(1-p_j^{(t)})}\Big/\left[c_j^{(t+1)}-\ln(1-p_j^{(t)})\right] \label{eq:p}
\eeq
for $t=1,\ldots,T$, one may connect the  observed (if $t=1$) or some latent (if $t\ge 2$) counts  $\xv_j^{(t)}\in\mathbb{Z}^{K_{t-1}}$ to the product $\Phimat^{(t)}\thetav_j^{(t)}$ at layer $t$  under the Poisson likelihood  as 
\beq
\xv_j^{(t)}\sim\emph{\mbox{Pois}}\left[-\Phimat^{(t)}\thetav_j^{(t)}\ln\left(1-p_j^{(t)}\right)\right].\label{eq:deepPFA_aug}
\eeq
\end{lem}

\begin{proof}
By definition (\ref{eq:deepPFA_aug}) is true for layer $t=1$.
Suppose that (\ref{eq:deepPFA_aug}) is true for layer $t\ge2$,
then we can augment each count  $x^{(t)}_{vj}$ into the summation of $K_{t}$  latent counts  that are smaller or equal as
\beq \textstyle
x^{(t)}_{vj}=\sum_{k=1}^{K_{t}} x^{(t)}_{vjk},~~x^{(t)}_{vjk}\sim\mbox{Pois}\left[-\phi_{vk}^{(t)}\theta_{kj}^{(t)}\ln\left(1-p_j^{(t)}\right)\right],\label{eq:PoAug}
\eeq
where $v\in\{1,\ldots,K_{t-1}\}$.
With $m^{(t)(t+1)}_{kj}:=x^{(t)}_{\cdotv jk} := \sum_{v=1}^{K_{t-1}}x^{(t)}_{vjk}$ representing the number of times that factor $k\in\{1,\ldots,K_t\}$ of layer $t$ appears in observation $j$ and $\mv^{(t)(t+1)}_{j}: = \big(x^{(t)}_{\cdotv j1},\ldots,x^{(t)}_{\cdotv jK_{t}}\big)'$, since $\sum_{v=1}^{K_{t-1}}\phi_{vk}^{(t)} = 1$, we can marginalize  out $\Phimat^{(t)}$ as in \cite{BNBP_PFA_AISTATS2012}, leading to 
\beq
\mv^{(t)(t+1)}_{j}\sim\mbox{Pois}\left[-\thetav_j^{(t)}\ln\left(1-p_j^{(t)}\right)\right]. \notag
\eeq
Further marginalizing  out  the gamma distributed $\thetav_j^{(t)}$ from the above Poisson likelihood leads to
\beq
\mv^{(t)(t+1)}_{j} \sim\mbox{NB}\left(\Phimat^{(t+1)}\thetav_j^{(t+1)}, p_j^{(t+1)}\right). \label{eq:NBAug}
\eeq
The $k$th element of $\mv^{(t)(t+1)}_{j}$ can be augmented under its compound Poisson representation as 
\beqs
&m^{(t)(t+1)}_{ kj} = \sum_{\ell=1}^{x^{(t+1)}_{kj}} u_{\ell},~~u_{\ell}\sim\mbox{Log}(p_j^{(t+1)}), ~~
x_{kj}^{(t+1)}\sim\mbox{Pois}\left[-\phiv_{k:}^{(t+1)}\thetav_j^{(t+1)}\ln\left(1-p_j^{(t+1)}\right)\right]. \notag
\eeqs
Thus if (\ref{eq:deepPFA_aug})  is true for layer $t$, then 
it is also true for layer $t+1$. 
\end{proof}
\begin{cor} [Propagate the latent counts upward] \label{cor:PGBN} Using Lemma 4.1 of \cite{BNBP_PFA_AISTATS2012} on (\ref{eq:PoAug}) and Theorem 1 of \cite{NBP2012} on (\ref{eq:NBAug}), we can propagate the  latent counts $x^{(t)}_{vj} $ of layer \emph{$t$}  upward to layer \emph{$t+1$} as
\begin{align}
& \textstyle
\left\{\left(x^{(t)}_{vj1},\ldots,x^{(t)}_{vjK_{t}}\right)\,\Big|\,x^{(t)}_{vj}, \phiv_{v:}^{(t)}, \thetav_j^{(t)}\right\}\sim\emph{\mbox{Mult}}\left(x^{(t)}_{vj}, \frac{\phi^{(t)}_{v1}\theta^{(t)}_{1j}}{\sum_{k=1}^{K_{t}}\phi^{(t)}_{vk}\theta^{(t)}_{kj}},\ldots,\frac{\phi^{(t)}_{vK_{t}}\theta^{(t)}_{K_{t}j}}{\sum_{k=1}^{K_{t}}\phi^{(t)}_{vk}\theta^{(t)}_{kj}}\right), \label{eq:step1}\\
&~~~~~~~~~~~~~~~\left( \left.x^{(t+1)}_{kj} \,\right |\,m^{(t)(t+1)}_{kj}, \phiv_{k:}^{(t+1)}, \thetav_j^{(t+1)}\right) \sim \emph{\mbox{CRT}}\left(m^{(t)(t+1)}_{kj}, \phiv_{k:}^{(t+1)}\thetav_{j}^{(t+1)}\right).\label{eq:CRT}
\end{align}
\end{cor}
As $x^{(t)}_{\cdotv j} = m^{(t)(t+1)}_{\cdotv j}$ and $x^{(t+1)}_{kj}$ is in the same order as  $\ln\big( m^{(t)(t+1)}_{kj}\big)$, 
the total count of layer $t+1$, expressed as  $\sum_j x^{(t+1)}_{\cdotv j}$,  would often be much smaller than that of layer $t$, expressed as $\sum_j x^{(t)}_{\cdotv j}$. Thus the PGBN may use $\sum_j x^{(T)}_{\cdotv j}$ as a simple criterion to decide whether to add more layers. 
\subsection{Modeling overdispersed counts} 
In comparison to a single-layer shallow model with $T=1$ that assumes the hidden units of layer one to be independent in the prior, the multilayer deep model with $T\ge 2$ captures the correlations between them. 
Note that for the extreme case that $\Phimat^{(t)}=\Imat_{K_{t}}$ for $t\ge 2$ are all identity matrices, which indicates that there are no correlations between the features of $\thetav_j^{(t-1)}$ left to be captured, 
the deep structure could still provide benefits as it helps 
model latent counts $\mv_{j}^{(1)(2)}$ that may be highly overdispersed.  For example, supposing $\Phimat^{(t)}=\Imat_{K_{2}}$ for all $t\ge2$, then from (\ref{eq:PGBN}) and (\ref{eq:NBAug}) we have 
\beq
m_{kj}^{(1)(2)}\sim\mbox{NB}(\theta_{kj}^{(2)},p_j^{(2)}), ~\ldots,~\theta_{kj}^{(t)}\sim\mbox{Gam}(\theta_{kj}^{(t+1)},1/c_j^{(t+1)}),~\ldots,~\theta_{kj}^{(T)}\sim\mbox{Gam}(r_k,1/c_j^{(T+1)}).\notag
\eeq
For simplicity, let us further assume $c_j^{(t)} = 1$ for all $t\ge 3$. Using the laws of total expectation and  total  variance,  we have 
$
\E\big[\theta_{kj}^{(2)}\,|\,r_k\big] = r_k$ and $
\mbox{Var}\big[\theta_{kj}^{(2)}\,|\,r_k\big] = (T-1) r_k,
$
and hence
$$
\E\big[m_{kj}^{(1)(2)}\,|\,r_k\big] = {r_k  p_j^{(2)}}/{(1-p_j^{(2)})},~~ 
\mbox{Var}\big[m_{kj}^{(1)(2)}\,|\,r_k\big] = {r_k  p_j^{(2)}}{\big(1-p_j^{(2)}\big)^{-2}}  \left[1+ (T-1)p_j^{(2)}\right]. 
$$
In comparison to   PFA with $m_{kj}^{(1)(2)}\,|\,r_k\sim\mbox{NB}(r_k,p_j^{(2)})$, 
with a variance-to-mean ratio of $1/(1- p_j^{(2)})$, 
the PGBN with $T$ hidden layers, which mixes the shape of $m_{kj}^{(1)(2)}\sim\mbox{NB}(\theta_{kj}^{(2)},p_j^{(2)})$ with a chain of gamma random variables, 
increases  the variance-to-mean ratio of the latent count $m_{kj}^{(1)(2)}$ given $r_k$  by a factor of $1+ (T-1)p_j^{(2)}$, and hence could better model highly overdispersed counts. 

\subsection{Upward-downward Gibbs sampling}\label{sec:sampling}
\vspace{-1.5mm}
With Lemma \ref{lem:PGBN} and Corollary \ref{cor:PGBN} and the width of the first layer being bounded by $K_{1\max}$, 
we develop an upward-downward Gibbs sampler for the PGBN, each iteration of which proceeds as follows: 
\emph{\textbf{Sample $x_{vjk}^{(t)}$}}. 
We can sample $x_{vjk}^{(t)}$ for all layers
using (\ref{eq:step1}). But for the first hidden layer, we may treat each observed count $x_{vj}^{(1)}$ as a sequence of 
word tokens at the $v$th term (in a vocabulary of size $V:=K_0$) in the $j$th document, and assign the $x_{\cdotv j}^{(1)}$ words $\{v_{ji}\}_{i=1,x_{\cdotv j}^{(1)}}$ one after another to the latent factors (topics),  with both the topics $\Phimat^{(1)}$ and topic weights $\thetav_j^{(1)}$ marginalized out, as
\beqs
&P(z_{ji}=k\,|\,-) 
\propto
\frac{\eta^{(1)}+x_{v_{ji}\cdotv k}^{(1)^{-ji}}}{V\eta^{(1)}+ x_{\cdotv \cdotv k}^{(1)^{-ji}}} \left(x_{\cdotv j k}^{(1)^{-ji}}+\phiv^{(2)}_{k:} \thetav_j^{(2)}\right),~~~k\in\{1,\ldots,K_{1\max}\}, \label{eq:z}
\eeqs
where $z_{ji}$ is the topic index for  $v_{ji}$ and $x_{vjk}^{(1)} := \sum_{i}\delta(v_{ji}=v,z_{ji}=k)$ counts the number of times that term $v$ appears in document~$j$; we use the $\cdotv$ symbol to represent summing over the corresponding index, \emph{e.g.},  $x^{(t)}_{\cdotv j k}:=\sum_v x^{(t)}_{vjk} $, and use $x^{-ji}$ to denote the count $x$ calculated without considering  word $i$ in  document $j$. The collapsed Gibbs sampling update equation shown above is related to the one developed in \cite{FindSciTopic}  for latent Dirichlet allocation, and the one developed in \cite{BNBP_EPPF} for PFA using the beta-negative binomial process. When $T=1$, 
we would replace the terms $\phiv^{(2)}_{k:} \thetav_j^{(2)}$ with $r_k$ for PFA built on the gamma-negative binomial process \cite{NBP2012} (or with $\alpha \pi_k$ for the hierarchical Dirichlet process latent Dirichlet allocation, see \cite{HDP} and \cite{BNBP_EPPF} for details), 
and add an additional term to account for the possibility of creating an additional topic \cite{BNBP_EPPF}. For simplicity, in this paper, we  truncate the nonparametric Bayesian model with $K_{1\max}$ factors and let $r_k\sim\mbox{Gam}(\gamma_0/K_{1\max},1/c_0)$ if $T=1$.     \\
\emph{\textbf{Sample $\phiv_k^{(t)}$}}. Given these latent counts, 
we sample the factors/topics $\phiv^{(t)}_k$ as
\beq
(\phiv^{(t)}_k\,|\,-)\sim\mbox{Dir}\left( \eta^{(t)}+ x^{(t)}_{1\cdotv k},\ldots, \eta^{(t)}+ x^{(t)}_{K_{t-1} \cdotv k} \right).\label{eq:step2}
\eeq
\emph{\textbf{Sample $x_{vj}^{(t+1)}$}}. We sample $\xv_j^{(t+1)}$ using (\ref{eq:CRT}), 
replacing  $\Phimat^{(T+1)}\thetav_j^{(T+1)}$ 
 with $\rv:=(r_1,\ldots,r_{K_{T}})'$.\\
\emph{\textbf{Sample $\thetav_j^{(t)}$}}. 
Using (\ref{eq:deepPFA_aug}) and the gamma-Poisson conjugacy, we sample $\thetav_j$ as
\beqs
&(\thetav_j^{(t)}\,|\,-)\sim\mbox{Gamma}\Big(\Phimat^{(t+1)}\thetav_j^{(t+1)} + \mv_j^{(t)(t+1)},\left[{c_j^{(t+1)}-\ln\left(1-p_j^{(t)}\right)}\right]^{-1}\Big). \label{eq:step4}
\eeqs
\emph{\textbf{Sample $\rv$}}. Both $\gamma_0$ and $c_0$ are sampled using related equations in \cite{NBP2012}. We sample $\rv$ as 
\beqs
&(r_v\,|\,-)\sim\mbox{Gam}\Big({\gamma_0}/K_T+x_{v\cdotv}^{(T+1)},\left[{c_0-\textstyle \sum_j\ln\big(1-p_j^{(T+1)}\big)}\right]^{-1}\Big).  \label{eq:step5}
\eeqs
\emph{\textbf{Sample $c_j^{(t)}$}}.  With $\theta_{\cdotv j}^{(t)}:=\sum_{k=1}^{K_{t}}\theta_{kj}^{(t)}$ for $t\le T$ and $\theta_{\cdotv j}^{(T+1)}:=r_{\cdotv}$, we sample $p_j^{(2)}$ and  $\{c_j^{(t)}\}_{t\ge 3}$  as
\begin{align}
&(p_j^{(2)}\,|\,-)\sim\mbox{Beta}\left(a_0\!+\! m_{\cdotv j}^{(1)(2)},b_0\!+\!\theta_{\cdotv j}^{(2)}\!\right),~(c_j^{(t)}\,|\,-)\sim\mbox{Gamma}\Big(e_0 \!+\! \theta_{\cdotv j}^{(t)}, \left[{f_0\!+\!\theta_{\cdotv j}^{(t-1)}}\!\right]^{-1}\!\Big),  \label{eq:step6}
\end{align}
and calculate $c_j^{(2)}$ and $\{p_j^{(t)}\}_{t\ge3}$ with (\ref{eq:p}).

\setlength{\textfloatsep}{7pt}
\begin{algorithm}[t]

\footnotesize
  \caption{\footnotesize The PGBN upward-downward Gibbs sampler that uses a layer-wise training strategy to train a set of networks, each of which adds an additional hidden layer on top of the  previously  inferred network, retrains all its layers jointly, and prunes inactive factors from the last layer. 
   \textbf{Inputs:} observed counts $\{x_{vj}\}_{v,j}$, upper bound of the width of the first layer $K_{1\max}$, upper bound of the number of layers $T_{\max}$, and hyper-parameters.  
    \textbf{Outputs:} A total of $T_{\max}$ jointly trained  PGBNs with depths $T=1$, $T=2$, $\ldots$, and $T=T_{\max}$.
  }\label{tab:algorithm}
  \begin{algorithmic}[1] 
   \For{\text{$T=1,2,\ldots,T_{\max}$}} Jointly train all the $T$ layers of the network
   \State
   Set $K_{T-1}$, the inferred width of layer $T-1$, as $K_{T\max}$, the upper bound of layer $T$'s width. 
    \For{\text{$iter=1: B_T+C_T$ 
    }} Upward-downward Gibbs sampling
     \State 
        \text{Sample $\{z_{ji}\}_{j,i}$ using collapsed inference;  Calculate $\{x_{vjk}^{(1)}\}_{v,k,j}$};
         \text{Sample $\{x_{vj}^{(2)}\}_{v,j}$} ;
      \For{\text{$t=2,3,\ldots,T$}}
        \State 
        \text{Sample $\{x_{vjk}^{(t)}\}_{v,j,k}$} ; 
         \text{  Sample $\{\phiv_k^{(t)}\}_{k}$} ;
        \text{  Sample $\{x_{vj}^{(t+1)}\}_{v,j}$} ;
      \EndFor
      \State
        \text{Sample $p_j^{(2)}$ and Calculate $c_j^{(2)}$}; Sample $\{c^{(t)}_j\}_{j,t}$  and Calculate $\{p^{(t)}_j\}_{j,t}$ for $t=3,\ldots,T+1$
        \For{\texttt{$t=T,T-1,\ldots,2$}}
      \State 
       \text{Sample $\rv$} if $t=T$;
        \text{Sample $\{\thetav_j^{(t)}\}_j$} ;
      \EndFor
     
%
%
      \If{ $iter=B_T$}
      \State 
      \!Prune layer $T$'s inactive factors $\{\phiv_{k}^{(T)}\}_{k:x_{\cdotv \cdotv k}^{(T)}=0}$, let $K_T=\sum_k {\delta(x_{\cdotv \cdotv k}^{(T)}>0)}$, and 
        update $\rv$; 

      \EndIf
    \EndFor

    \State
    Output the posterior means (according to the last MCMC sample) of  all remaining  factors 
    $\{\phiv_k^{(t)}\}_{k,t}$
    as the inferred network of $T$ layers, and $\{r_k\}_{k=1}^{K_T}$ as 
    the gamma shape parameters of layer $T$'s 
    hidden units. 
     \EndFor
  \end{algorithmic}
  \normalsize
\end{algorithm}%

\vspace{-1mm}
\subsection{Learning the network structure with layer-wise training}
\vspace{-1mm}
As jointly training all layers together is often difficult, existing deep networks are typically trained 
using a greedy layer-wise unsupervised training algorithm,  such as the one proposed in \cite{hinton2006fast} to train the deep belief networks.  The effectiveness of this training strategy is further analyzed in \cite{bengio2007greedy}. 
By contrast, the PGBN has a simple 
Gibbs sampler to  jointly train all its hidden layers, as described in Section \ref{sec:sampling}, and 
hence does not require greedy layer-wise training. Yet the same as commonly used deep learning algorithms, it still needs to specify the number of layers and the width of each layer. 

In this paper, we adopt the idea of  layer-wise training for the PGBN, not because of the lack of an effective joint-training algorithm, but for the purpose of learning the width of each hidden layer in a greedy layer-wise manner, given a fixed budget on the 
width of the first layer. The proposed layer-wise training strategy is summarized in Algorithm~\ref{tab:algorithm}. With a  PGBN of $T-1$ layers that has already been trained, the key idea  is to use a truncated gamma-negative binomial process \cite{NBP2012} to model the latent count matrix for the newly added top layer as $m_{kj}^{(T)(T+1)}\sim\mbox{NB}(r_k,p_j^{(T+1)}),~ r_k\sim\mbox{Gam}(\gamma_0/K_{T\max},1/c_0)$, and rely on that stochastic process's shrinkage mechanism 
 to prune inactive factors (connection weight vectors) of layer $T$, and hence the inferred $K_T$ would be smaller than $K_{T\max}$ if $K_{T\max}$ is sufficiently large. The newly added layer and the layers below it  would be jointly trained, but with 
the structure  below the newly added layer kept unchanged. 
 Note that when $T=1$, the PGBN would infer the number of active factors if $K_{1\max}$ is set large enough, otherwise, it would still assign the factors with different weights $r_k$, but may not be able to prune any of them.

\vspace{-3.5mm}
\section{Experimental Results}
\vspace{-2.5mm}

We apply the PGBNs for topic modeling of text corpora, each document of which 
is represented as a term-frequency 
count vector. 
Note that the PGBN with a single hidden layer is identical to the (truncated) gamma-negative binomial process PFA of  \cite{NBP2012}, which is a nonparametric Bayesian algorithm that performs similarly  to the hierarchical Dirichlet process latent Dirichlet allocation \cite{HDP} for text analysis, and is considered as a strong baseline that outperforms a large number of topic modeling algorithms. Thus we will focus on making comparison to the PGBN with a single layer, with its layer width set to be large to approximate the performance of the gamma-negative binomial process PFA.
We evaluate the PGBNs' performance by examining both how well they unsupervisedly extract low-dimensional features for document classification, and how well they predict heldout word tokens. Matlab code will be available  in \href{http://mingyuanzhou.github.io/}{http://mingyuanzhou.github.io/}.


We use Algorithm \ref{tab:algorithm} to learn,  in a layer-wise manner,  from the training data  the  weight matrices $\Phimat^{(1)},\ldots,\Phimat^{(T_{\max})}$ and  the top-layer hidden units' gamma shape parameters $\rv$: 
to add layer $T$ to a previously trained network with $T-1$ layers, we use $B_T$ 
 iterations to jointly train $\Phimat^{(T)}$ and $\rv$ together with $\{\Phimat^{(t)}\}_{1,T-1}$, prune the inactive factors of layer $T$, and continue the joint training with another $C_T$ iterations.  
We set the hyper-parameters as $a_0=b_0=0.01$ and $e_0=f_0=1$. 
Given the trained network, we apply the upward-downward Gibbs sampler to collect 500 MCMC  samples after 500 burnins to estimate the posterior mean of the feature usage proportion vector $\thetav_j^{(1)}/\theta_{\cdotv j}^{(1)}$ at the first hidden layer, for every document in both the training and testing sets.

\begin{figure}[!tb]
\begin{center}
\includegraphics[width=56mm]{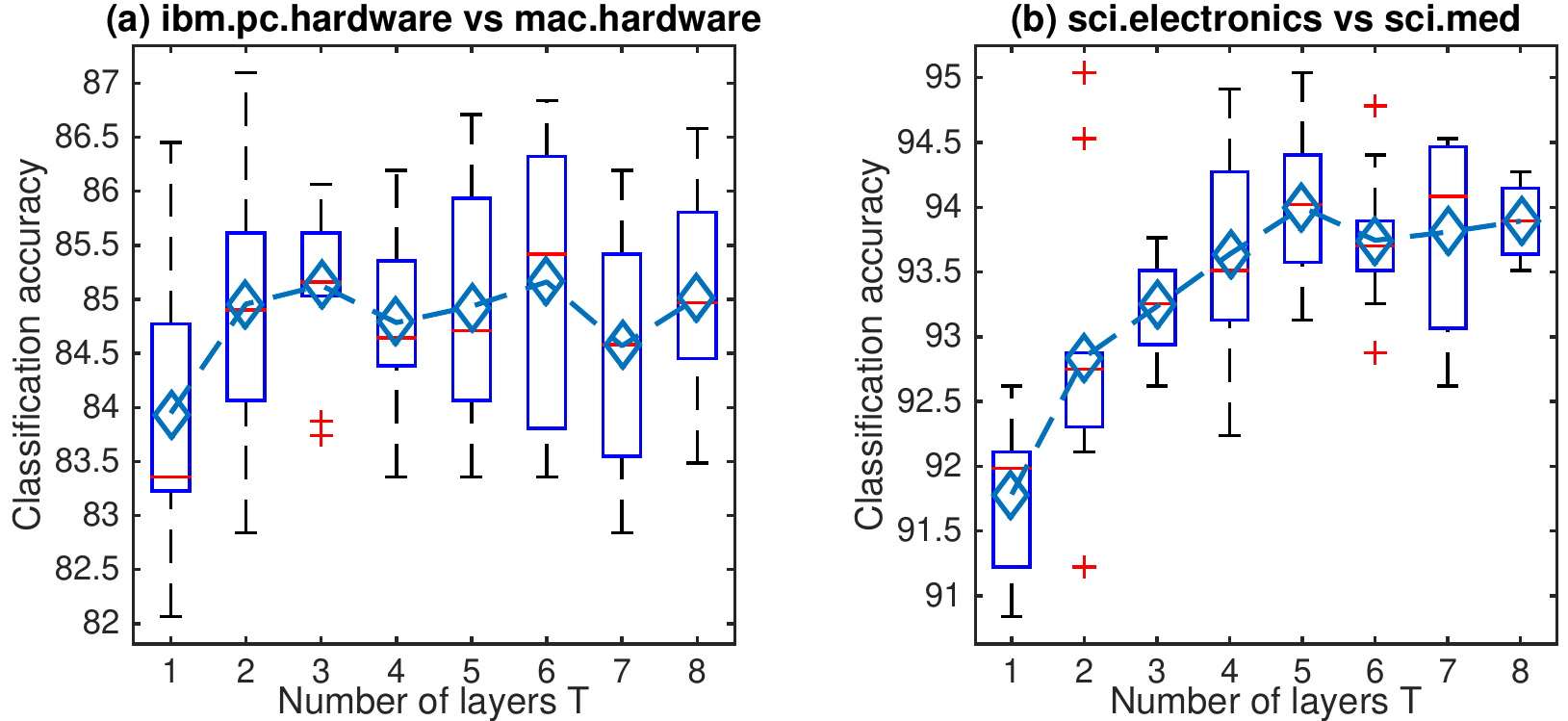}\, \, \,
\includegraphics[width=56mm]{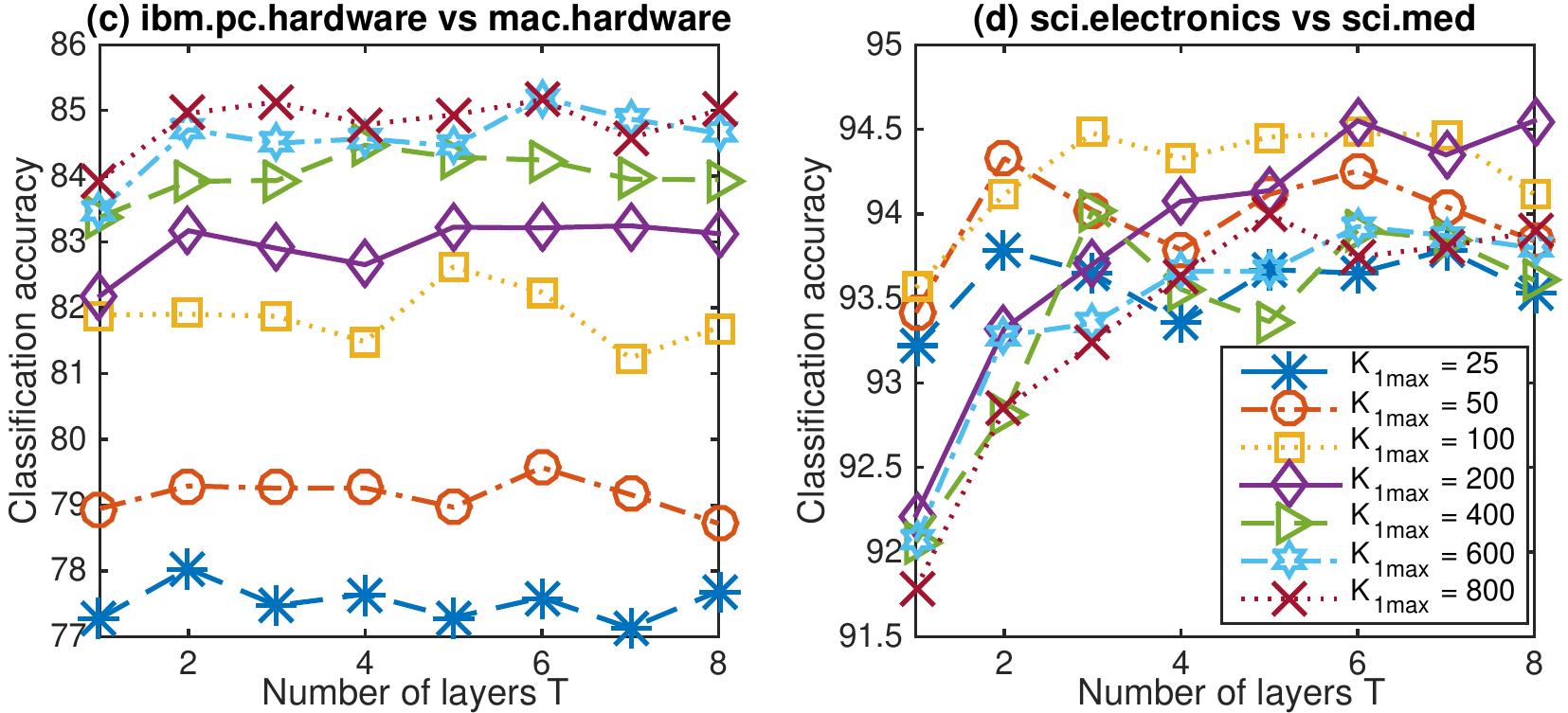}
\end{center}
\vspace{-5.5mm}
\caption{\small \label{fig:binary_20news}
Classification accuracy (\%) as a function of the network depth $T$ for two  20newsgroups binary classification tasks, 
with $\eta^{(t)} =0.01$ for all layers. 
(a)-(b): the boxplots of the accuracies of 12 independent runs with $K_{1\max}=800$. (c)-(d): the average accuracies of these 12 runs for various $K_{1\max}$ and $T$.
Note that $K_{1\max}=800$ is large enough to cover all active first-layer topics (inferred to be around 500 for both binary classification tasks), whereas all the first-layer topics would be used if $K_{1\max}=25, 50, 100$, or $200$.
 \vspace{0.mm}
}

\begin{center}
\includegraphics[width=49mm]{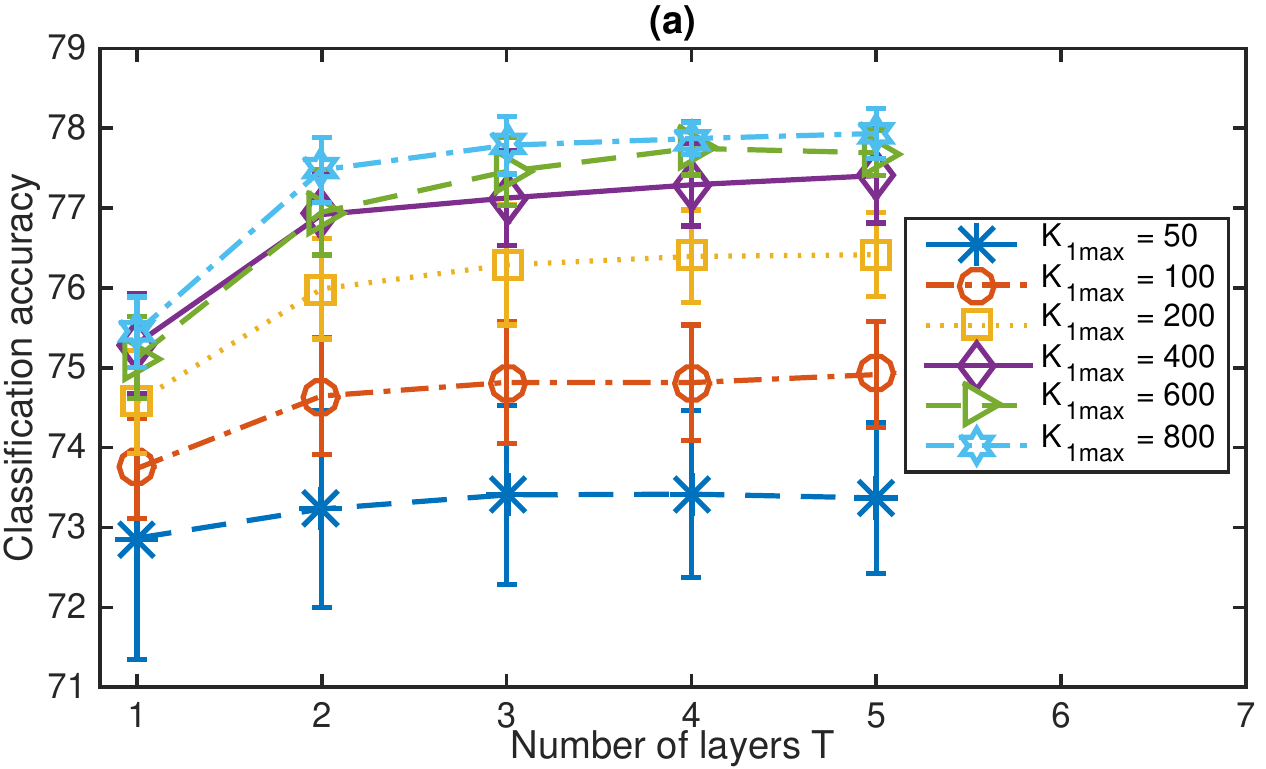}
\,\,
\includegraphics[width=49mm]{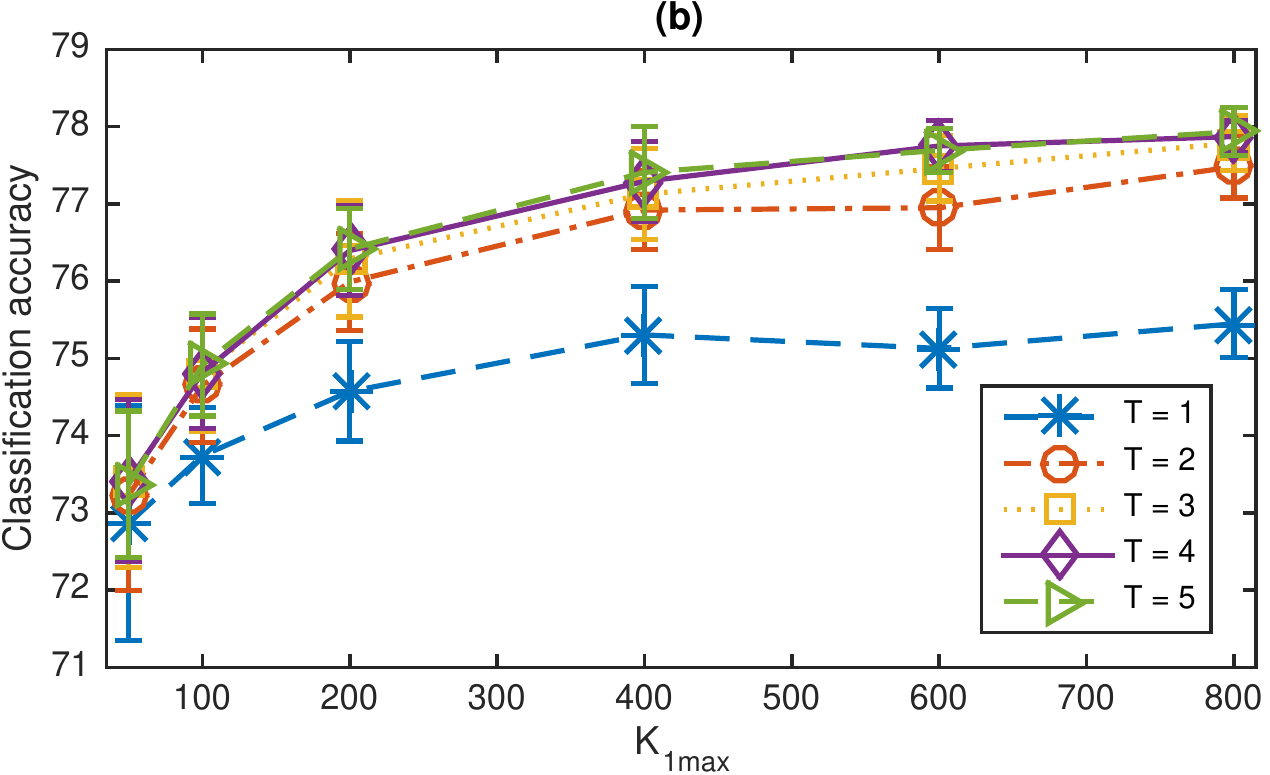}
\end{center}
\vspace{-5.3mm}
\caption{\small \label{fig:full_20news}
Classification accuracy (\%) of the PGBNs for 20newsgroups multi-class classification (a) as a function of the depth $T$ with various $K_{1\max}$ and (b)  as a function of $K_{1\max}$ with various depths, 
with $\eta^{(t)}=0.05$ for all layers. 
The widths of hidden layers are automatically inferred, with $K_{1\max}=50$, $100$, $200$, $400$, $600$, or $800$.  
Note that 
$K_{1\max}=800$  is large enough to cover all active first-layer topics, whereas all the first-layer topics would be used if $K_{1\max}=50$, $100$, or $200$.
 \vspace{0.mm}
}
\end{figure}

\textbf{Feature learning for binary classification.}
We consider the 20 newsgroups dataset ({\href{http://qwone.com/~jason/20Newsgroups/}{http://qwone.com/$\sim$jason/20Newsgroups/}}) that consists of 18,774 documents from 20 different news groups, with a vocabulary of size $K_0= $ 61,188. It is partitioned into a training set of 11,269 documents and a testing set of 7,505 ones. 
We first consider 
two binary classification tasks that distinguish between 
the $comp.sys.ibm.pc.hardware$ and  $comp.sys.mac.hardware$, and between the
$sci.electronics$ and $sci.med$ news groups.  
For each binary classification task, we remove  a standard list of stop words and only consider the terms that appear at least five times, 
and report the classification accuracies based on 12 independent random trials. 
With the upper bound of the first layer's width set as $K_{1\max}\in\{25,50,100, 200, 400, 600, 800\}$, and $B_t=C_t=1000$ and $\eta^{(t)}=0.01$  
for all $t$, we use Algorithm \ref{tab:algorithm} to train a network with $T\in\{1, 2, \ldots,8\}$ layers. Denote~$\bar{\thetav}_j$ as the estimated $K_1$ dimensional feature vector for document $j$, where $K_1\le K_{1\max}$ is the inferred number of active factors of the first layer that is bounded by the pre-specified truncation level $K_{1\max}$. We use the $L_2$ regularized logistic regression  provided by the LIBLINEAR package \cite{REF08a} to train a linear classifier on $\bar{\thetav}_j$ in the training set and use it to classify $\bar{\thetav}_j$ in the test set, where the regularization parameter is five-folder 
 cross-validated on the training set from $(2^{-10}, 2^{-9},\ldots, 2^{15})$. 

As shown in Fig.~\ref{fig:binary_20news}, modifying the PGBN from a single-layer shallow network to a multilayer deep one clearly improves the qualities of the unsupervisedly extracted 
feature vectors. 
In a random trial, with $K_{1\max}=800$, we infer a network structure of  $(K_1,\ldots,K_8)=(512,   154,    75,    54 ,   47  ,  37,    34 ,   29)$ for the first binary classification task, and $(K_1,\ldots,K_8)=(491, 143,    74,    49,    36,    32,    28 ,   26)$ for the second one. 
 Figs.~\ref{fig:binary_20news}(c)-(d) also show that  increasing the network depth  in general improves the performance, but the first-layer width clearly plays an important role in controlling the ultimate network capacity. This insight is further illustrated below. 

\textbf{Feature learning for multi-class classification.} We test the PGBNs for multi-class classification on  20newsgroups. After removing a standard list of  stopwords and the terms that appear less than five times, we obtain a vocabulary with $K_0= 33,420$. We set 
 $C_t=500$ and $\eta^{(t)} = 0.05$ for all $t$. If $K_{1\max}\le 400$, we set $B_t=1000$ for all $t$, otherwise we  set $B_1=1000$ and $B_t=500$ for $t\ge 2$. We use all 11,269 training documents to infer a set of networks with  $T_{\max}\in\{1,\ldots,5\}$ and $K_{1\max}\in \{50, 100, 200, 400, 600,800\}$, and mimic the same testing procedure used for binary classification  to extract low-dimensional  feature vectors, with which each testing document is classified to one of the 20 news groups using the $L_2$ regularized  logistic regression. 
  Fig. \ref{fig:full_20news} shows a clear trend of improvement in classification accuracy by increasing the network depth with a limited first-layer width, or by increasing the upper bound of the width of the first layer with the depth fixed. For example, a single-layer PGBN with $K_{1\max}=100$ could add one or more layers to slightly outperform a single-layer PGBN with $K_{1\max}=200$, and a single-layer PGBN with $K_{1\max} =200$ could add layers to clearly outperform a single-layer PGBN with $K_{1\max}$ as large as $800$. We also note that 
each iteration of  jointly training multiple layers costs moderately more than that of training a single layer, 
   e.g., with $K_{1\max}=400$, a training iteration on a single core of an Intel Xeon 2.7 GHz CPU on average takes about $5.6$, 
$6.7$, $7.1$ seconds for the PGBN with $1$, $3$, and $5$ layers, respectively.

 Examining the inferred network structure also reveals interesting details. For example, in a random trial with Algorithm 1, 
the inferred network widths $(K_1,\ldots,K_5)$ are 
$( 50  ,  50   , 50   , 50,    50)$, 
$(200,   161  , 130  ,  94,    63)$,
$(528 ,  129   ,109 ,   98,    91)$, and
$(608,   100,    99,    96,    89)$,
 for  $K_{1\max}=50,200,600$, and $800$, respectively. 
This indicates  that for a network with an insufficient budget on its first-layer width, as the network depth increases, its inferred layer widths decay 
 more slowly than a network with a sufficient or surplus budget on its first-layer width; and a network with a surplus budget on its first-layer width may only need relatively small widths for its higher hidden layers. In the Appendix, we provide comparisons of 
 accuracies  between the PGBN and other related algorithms, including these of \cite{larochelle2012neural} and  \cite{srivastava2013modeling}, on similar multi-class document classification tasks.

\begin{figure}[!tb]
\begin{center}
\includegraphics[width=108mm]{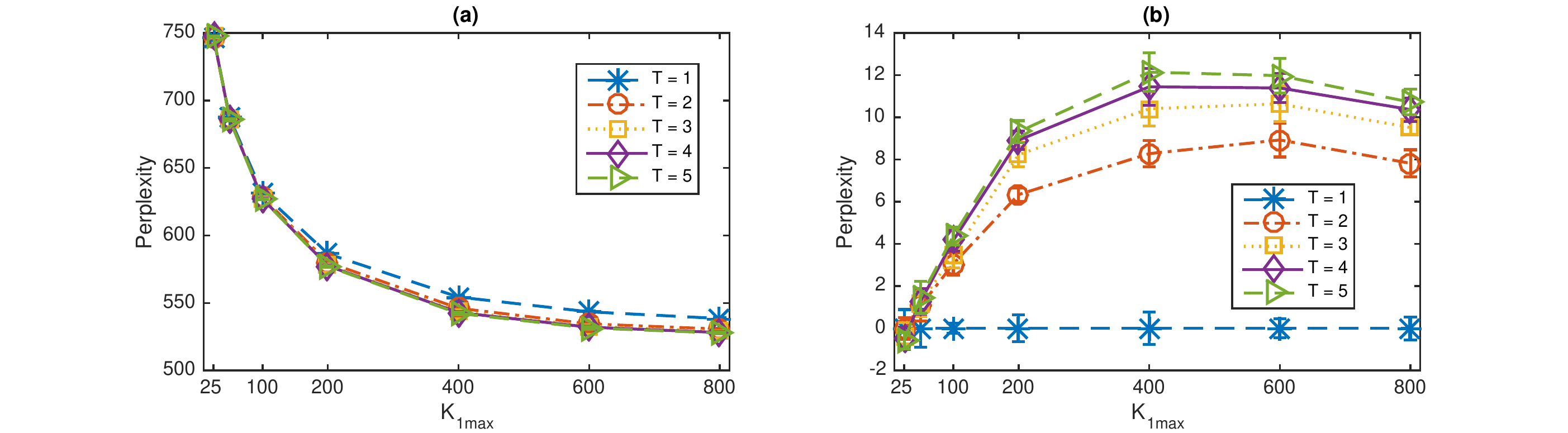}
\end{center}
\vspace{-5.3mm}
\caption{\small \label{fig:Perplexity}
(a)  per-heldout-word perplexity (the lower the better) for the NIPS12 corpus (using the 2000 most frequent terms) as a function of the upper bound of the first layer width $K_{1\max}$ and network depth $T$, 
with $30\%$ of the word tokens in each document used for training and $\eta^{(t)}=0.05$ for all $t$. 
(b) for visualization, each curve in (a) is reproduced by subtracting its values from the average perplexity of the single-layer network. 
\vspace{.5mm}
}
\end{figure}


\textbf{Perplexities for holdout words.}
In addition to examining the performance of the PGBN for unsupervised feature learning, we also consider a more direct approach that we randomly choose 30\% of the word tokens in each document as training, and use the remaining ones to calculate per-heldout-word  perplexity. We consider 
the NIPS12 (\href{http://www.cs.nyu.edu/~roweis/data.html}{http://www.cs.nyu.edu/$\sim$roweis/data.html})  corpus,
 limiting the vocabulary to the 2000 most frequent terms.
We set $\eta^{(t)}=0.05$ and $C_t=500$ for all $t$, 
set $B_1=1000$ and $B_t=500$ for $t\ge2$, and consider five random trials.  
 Among the $B_t+C_t$ Gibbs sampling iterations used to train layer $t$, we collect one sample per  five iterations during the last $500$ iterations, for each of which we draw the topics $\{\phiv^{(1)}_k\}_k$ and  topics weights  $\thetav_{j}^{(1)}$, to compute the per-heldout-word perplexity using Equation (34) of \cite{NBP2012}. As shown in Fig. \ref{fig:Perplexity}, we observe a clear trend of improvement by increasing both $K_{1\max}$ and $T$.

\textbf{Qualitative analysis and document simulation.} In addition to these quantitative experiments,  we have also examined the topics learned at each layer. We use $\big(\prod_{\ell=1}^{t-1} \Phimat^{(\ell)}\big) \phiv_{k}^{(t)}$ to project topic $k$ of layer $t$ as a $V$-dimensional word probability vector. Generally speaking, the topics at lower layers are more specific, whereas those at higher layers are more general. 
E.g., examining the results used to produce Fig. \ref{fig:Perplexity}, 
with $K_{1\max}=200$ and $T=5$, the PGBN infers a network with $(K_1,\ldots,K_5)=$ $
(200, 164 ,106, 60 ,42)$. 
The ranks (by popularity) and top five words of  three example topics for layer $T=5$ are
``6 network units input learning training,''
``15 data model learning set image,'' and
``34 network learning model input neural;''
while these of five example topics of layer $T=1$ are 
``19 likelihood em mixture parameters data,''
``37 bayesian posterior prior log evidence,''
``62 variables belief networks conditional inference,''
``126 boltzmann binary machine energy hinton,''
and
``127 speech speaker acoustic vowel phonetic.''
We have also tried drawing $\thetav^{(T)}\sim\mbox{Gam}\big(\rv,1/c_j^{(T+1)}\big)$ and downward passing it through the $T$-layer network  to generate synthetic documents, which are found to be quite interpretable and reflect various general aspects of the corpus used to train the network. 
We provide in the Appendix a number of synthetic documents generated from a  PGBN trained on the 
20newsgroups corpus, whose inferred structure is $(K_1,\ldots,K_5)=(608,   100,    99,    96,    89)$. 
\vspace{-3mm}
\section{Conclusions}
\vspace{-2mm}
The Poisson gamma belief network is proposed to extract a multilayer deep representation for high-dimensional count vectors, with an efficient upward-downward Gibbs sampler to jointly train all its layers and a layer-wise training strategy to automatically infer the network structure. 
Example results clearly demonstrate the advantages of deep topic models.
For big data problems, in practice  one may rarely has a sufficient budget to allow the first-layer width to grow without bound, thus it is natural to consider a belief network that can use a 
deep representation to not only enhance its representation power, but also better allocate its computational resource. 
Our algorithm 
achieves a good compromise between the widths of hidden layers and the depth of the network.

\textbf{Acknowledgements.} 
M. Zhou thanks TACC for computational support. B. Chen thanks the support of the Thousand Young Talent Program of China, NSC-China (61372132), and NCET-13-0945. 
\newpage

\bibliographystyle{unsrt}
\bibliography{References102014}
\pagebreak

\normalsize
\appendix
\begin{center}
\textbf{\large Appendix for The Poisson Gamma Belief Network}
\end{center}
\setcounter{section}{0}
\section{Comparisons of classification accuracies}
For comparison, we consider the same $L_2$ regularized logistic regression 
multi-class classifier, trained either on the raw word counts or normalized term-frequencies of the 20newsgroups training documents using five-folder cross-validation.
As summarized in Tab. \ref{tab:LR}, when using the raw term-frequency word counts as covariates,  the same classifier achieves $69.8\%$ ($68.2\%$) accuracy on the 20newsgroups test documents if using the top 2000 terms that exclude (include) a standard list of stopwords, achieves $75.8\%$ if using all the $61,188$ terms in the vocabulary, and achieves $78.0\%$ if using the $33,420$  terms remained after removing a standard list of stopwords and the terms that appear less than five times; 
and when using the normalized term-frequencies as covariates, the corresponding accuracies are $70.8\%$ ($67.9\%$) if using the top 2000 terms excluding (including) stopwords, $77.6\%$ with all the $61,188$ terms, and $79.4\%$ with  the $33,420$  selected terms.

\begin{table}[h]
\begin{small}
\caption{\small Multi-class classification accuracy of $L_2$ regularized logistic regression.}\label{tab:LR}
\begin{center}
\begin{tabular}{ c c c c}
\toprule
 $V=61,188$ & $V=61,188$ & $V=33,420$ & $V=33,420$ \\
 with stopwords &  with stopwords & remove stopwords  & remove stopwords \\
 with rare words&  with rare words & remove rare words &  remove rare words\\
 raw word counts &  term frequencies & raw word counts &  term frequencies \\
\midrule
 75.8\% & 77.6\% & 78.0\% & 79.4\%\\
\bottomrule
\end{tabular}
\end{center}
%
%
\begin{center}
\begin{tabular}{ c c c c}
\toprule
 $V=2000$ & $V=2000$ & $V=2000$ & $V=2000$ \\
  with stopwords &  with stopwords & remove stopwords & remove stopwords \\
 raw counts &  term frequencies & raw counts &  term frequencies \\
\midrule
 68.2\% & 67.9\% & 69.8\% & 70.8\%\\
\bottomrule
\end{tabular}
\end{center}
\end{small}
\end{table}%

As summarized in Tab. \ref{tab:ORS}, for multi-class classification on the same dataset, with a vocabulary size of 2000 that consisits of the 2000 most frequent terms after removing stopwords and stemming,  the DocNADE \cite{larochelle2012neural} and the over-replicated softmax \cite{srivastava2013modeling}  provide the accuracies of $67.0\%$	and $66.8\%$, respectively, for a feature dimension of $K=128$, and  provide   the accuracies of $68.4\%$  and	$69.1\%$, respectively,  for a feature dimension of $K=512$.

\begin{table}[h]
\begin{small}
\caption{\small Multi-class classification accuracy of the DocNADE \cite{larochelle2012neural} and over-replicated softmax \cite{srivastava2013modeling}.}\label{tab:ORS}
\begin{center}
\begin{tabular}{ c | c  c}
\toprule
 &$V=2000$, $K=128$ & $V=2000$, $K=512$\\
 &  remove stopwords, stemming &  remove stopwords, stemming \\
\midrule
DocNADE & 67.0\% & 68.4\%\\
Over-replicated softmax &66.8\% & 69.1\%\\
\bottomrule
\end{tabular}
\end{center}
\end{small}
\end{table}%

As summarized in Tab. \ref{tab:PGBN}, with the same vocabulary size of 2000 (but different terms due to different preprocessing),
the proposed PGBN provides $65.9\%$ ($67.5\%$) with $T=1$ ($T=5$) for $K_{1\max}=128$,  and $65.9\%$ ($69.2\%)$ with $T=1$ ($T=5)$ for $K_{1\max}=512$, which may be further improved if we also consider the stemming step, as done in the these two algorithms, for word preprocessing, or if we set the values of  $\eta^{(t)}$ to be smaller than 0.05. We also summarize in Tab. \ref{tab:PGBN} the classification accuracies of the PGBNs learned  with $V=33,420$, as shown in Fig. \ref{fig:full_20news}.

\begin{table}[h]
\begin{small}
\caption{\small Classification accuracy of the PGBN trained with $\eta^{t}=0.05$ for all $t$.}\label{tab:PGBN}
%
%
%
\begin{center}
\begin{tabular}{ c | c  c c}
\toprule
 &$V=2000$, $K_{1\max}=128$  &$V=2000$, $K_{1\max}=256$  & $V=2000$, $K_{1\max}=512$\\
 &  remove stopwords& remove stopwords& remove stopwords \\
\midrule
PGBN ($T=1$) & $65.9\% \pm 0.4\%$& $66.3\% \pm0.4\%$ & $65.9\% \pm0.4\%$\\
PGBN ($T=5$) & $67.5\% \pm0.4\%$ & $68.8\% \pm0.3\%$ & $69.2\% \pm0.4\%$\\
\bottomrule
\end{tabular}
\end{center}
\begin{center}
\begin{tabular}{ c | c  c c}
\toprule
 &$V=33,420$, $K_{1\max}=200$  &$V=33,420$, $K_{1\max}=400$  & $V=33,420$, $K_{1\max}=800$\\
 &  remove stopwords& remove stopwords& remove stopwords \\
 &  remove rare words& remove rare words& remove rare words \\
\midrule
PGBN ($T=1$) & $74.6\% \pm 0.6\%$& $75.3\% \pm0.6\%$ & $75.4\% \pm0.4\%$\\
PGBN ($T=5$) & $76.4\% \pm0.5\%$ & $77.4\% \pm0.6\%$ & $77.9\% \pm0.3\%$\\
\bottomrule
\end{tabular}
\end{center}
\end{small}
\end{table}%

\section{Generating synthetic documents}
Below we provide several synthetic documents generated from the  PGBN  with $(K_1,\ldots,K_5)=(608,   100,    99,    96,    89)$, 
which is trained on the training set of the 
20newsgroups corpus with $K_{1\max}=800$ and $\eta^{(t)}=0.05$ for all $t$. 
We set $c_{j'}^{(t)}$ as the median of the inferred $\{c_j^{t}\}_j$ of the training documents for all $t$. Given $\{\Phimat^{(t)}\}_{1,T}$ and $\rv$, We first generate  $\thetav_{j'}^{(T)}\sim\mbox{Gam}\left(\rv,1\big/c_{j'}^{(T+1)}\right)$ and then downward pass it through the network by repeatedly drawing nonnegative real random variables from the gamma distribution as in 
\eqref{eq:PGBN}.
With the simulated $\thetav_{j'}^{(1)}$, we calculate the Poisson rates for all the $V$ words using $\Phimat^{(1)}\thetav_{j'}^{(1)}$ and 
display the top 100 words ranked according to $\Phimat^{(1)}\thetav_{j'}^{(1)}$. Below are some example synthetic  documents generated in this manner, 
which are all easy to interpret and reflect various aspects of the  20newsgroups corpus used to train the PGBN. 

\begin{itemize}
\item
team game games hockey year cup season playoffs edu win pittsburgh nhl toronto detroit stanley teams montreal play jets pens espn division chicago new penguins pick league players devils rangers wings boston islanders playoff ca series winnipeg gm abc tv playing quebec april time round st vancouver fans best gld bruins coach winner calgary leafs player great watch night patrick vs finals conference final just baseball coverage murray minnesota don won gary points mike like ice kings regular mario played louis caps contact washington selanne norris buffalo columbia keenan star people fan th think canadiens said canada canucks york gerald

\item
hall smith players fame career ozzie winfield nolan guys ryan dave baseball eddie murray numbers steve kingman robinson yount morris roger years bsu puckett long joe jackson hung brett garvey deserve robin evans princeton yeah frank ruth kirby rickey pitcher peak yogi hof great sick lee ha aaron johnny darrell santo time greatest stats seasons ron george reardon shortstops henderson hank mays jack liability marginal rogers average compare belong schmidt gibson willie leo ucs sgi bsuvc comment fans honestly deserves cal bell candidates wagner fielding walks ve likely history gee heck consideration mike player bonds lock rating sandberg standards apparent

\item
fbi koresh batf gas compound waco government atf people children tear cult davidians did bd branch agents happened assault warrant david reno tanks killed weapons clinton point country search building federal raid press started reported death proper needed illegal better house protect burned janet outside burn days media stand job arms inside right come cwru equipment followers investigation oldham believe non power kids burning fires women suicide law order cs sick blame initial alive feds agent tank religious automatic davidian deaths knock good hit said military possible died away light fault child witnesses pay instead folks daniel bureau armored going
\item
people government law state israel rights israeli jews right public states war fact political country arab laws article case court human federal american united support society policy civil freedom members national jewish evidence person majority force power legal citizens action crime world act countries issue arabs group police justice non control palestinian live land peace true anti center writes gaza population research constitution death edu org allowed party protection consider actions number adam apc general subject based murder igc considered life military self parties lives personal nation order cpr social question individual religious today situation free responsibility governments palestine innocent
%

\item
medical health disease doctor pain patients treatment medicine cancer edu hiv blood use years patient writes cause skin don like just aids symptoms number article help diseases drug com effects information doctors infection physician normal chronic think taking care volume condition drugs page says cure people tobacco hicnet know newsletter effective therapy problem common time women prevent surgery children center immune research called april control effect weeks low syndrome hospital physicians states clinical diagnosed day med age good make caused severe reported public safety child said cdc usually diet national studies tissue months way cases causing migraine smokeless infections does


\item
card video drivers cards driver vga mode ati graphics windows diamond vesa bus svga support gateway dx pc modes color isa board version local bit memory vlb ultra pro eisa monitor new does mb stealth hz using based speedstar orchid colors available latest ram know work chip performance resolution fast screen speed tech million trident winbench dcoleman set problems yes et ftp results winmarks plus edu bbs zeos utexas vram bios robert win higher magazine utxvms able high interlaced viper com boards site weitek tseng chipset modem turbo software non resolutions far faster accelerated supports price meg ega mhz true

%
\item 
card windows video drivers monitor com modem vga cards driver port pc mode screen ati serial graphics dos bus board irq support svga diamond vesa using memory problem dx color gateway file version ports local modes pro bit does isa colors mb know vlb mouse ultra win ram new monitors hz work eisa nec problems chip files stealth use set program speedstar orchid plus high based resolution fast software cable hardware display latest used performance ms like baud bbs tech connector run thanks speed just yes million trident winbench dcoleman available pin ibm uart connect sony window switch et disk
\item
nissan electronics wagon altima delcoelect kocrsv station gm subaru sumax delco spiros hughes wax pathfinder legacy kokomo wagons smorris scott toyota seattleu don just like strong silver software luxury derek proof stanza seattle cisco morris cymbal triantafyllopoulos sportscar think people know near fool ugly proud claims flat statistics lincoln sedans bullet karl lee perth puzzled miata sentra maxima acura infiniti corolla mgb untruth verbatim good time consider way based make stand guys writes noticed want ve heavy suggestion eat steven horrible uunet studies armor fisher lust designs study definately lexus remove conversion embodied aesthetic elvis attached honey stole designing wd

\item
mac apple bit mhz ram simms mb like memory just don cpu people chip chips think color board ibm speed does know se video time machines motherboard hardware lc cache meg ns simm need upgrade built vram good quadra want centris price dx run way processor card clock slots make fpu internal did macs cards ve pin power really machine say faster said software intel macintosh right week writes slot going sx performance things edu years nubus possible thing monitor work point expansion rom iisi ll add dram better little slow let sure pc ii didn ethernet lciii case kind

\item
image jpeg gif file color files images format bit display convert quality formats colors programs program tiff picture viewer graphics bmp bits xv screen pixel read compression conversion zip shareware scale view jpg original save quicktime jfif free version best pcx viewing bitmap gifs simtel viewers don mac usenet resolution animation menu scanner pixels sites gray quantization displays better try msdos tga want current black faq converting white setting mirror xloadimage section ppm fractal amiga write algorithm mpeg pict targa arithmetic export scodal archive converted grasp lossless let space human grey directory pictures rgb demo scanned old choice grayscale compress
\item
gun guns edu writes bike com article weapons dod control crime weapon apr used carry criminals police ride nra bikes self firearms use buy firearm laws concealed bmw defense home handgun criminal motorcycle anti problem car people owners ban rider riding shot just armed new don like crimes assault kill violent protect uio handguns ifi evil ama citizens state org know illegal politics texas thomas thomasp cb talk legal shooting pro road carrying abiding think att honda cs stolen defend good purchase ll law individual hp cc permit rifle issue government states parsli property ve killing federal does motorcycles time

\item
gun guns weapons people control government law crime state rights police laws weapon self criminals carry states public nra used defense firearms anti federal right criminal legal firearm citizens country home political case concealed handgun court fact crimes issue protect armed politics kill ban problem buy national individual support shot society violent use civil war property talk owners assault illegal handguns ifi uio united defend action allowed freedom article american amendment person member power force thomasp car human evidence threat thomas murder shooting majority killed carrying members citizen killing pro abiding group act evil texas america justice permit stolen said

\end{itemize}

\end{document}